\documentclass[sigconf]{acmart}

\copyrightyear{2026}
\acmYear{2026}
\setcopyright{cc}
\setcctype{by}
\acmConference[WWW '26] {Proceedings of the ACM Web Conference 2026}{April 13--17, 2026}{Dubai, United Arab Emirates.}
\acmBooktitle{Proceedings of the ACM Web Conference 2026 (WWW '26), April 13--17, 2026, Dubai, United Arab Emirates}
\acmISBN{979-8-4007-2307-0/2026/04} 
\acmDOI{10.1145/3774904.3792822}

\usepackage{mathtools}
\usepackage{amsthm}
\usepackage{multirow}
\usepackage{array}
\usepackage{graphicx}
\usepackage{pifont} 
\usepackage{bm}
\usepackage{algorithm}
\usepackage{algorithmic}
\usepackage{mathtools}
\usepackage{soul}
\usepackage[export]{adjustbox}
\theoremstyle{plain}
\newtheorem{theorem}{Theorem}[section]

\theoremstyle{definition}
\newtheorem{definition}[theorem]{Definition}

\theoremstyle{remark}

\usepackage{listings}
\usepackage{xcolor}
\usepackage{tcolorbox}
\definecolor{rq}{HTML}{1B365C}
\definecolor{rqBack}{HTML}{9ECBF7}
\usepackage{balance}

\definecolor{lightgrey}{rgb}{0.9,0.9,0.9}

\definecolor{MyBlue}{HTML}{a9d5ee}
\definecolor{MyLightBlue}{HTML}{DAEEFA}
\newcommand{\textcolorblue}[1]{
  \begingroup
  \sethlcolor{MyBlue}
  \textcolor{black}{\hl{#1}}
  \endgroup
}

\begin{document}

\title{Orchestration-Free Customer Service Automation: A Privacy-Preserving and Flowchart-Guided Framework}

\author{Mengze Hong}
\affiliation{%
  \institution{Hong Kong Polytechnic University}
  \city{Hong Kong}
  \country{China}
}

\author{Chen Jason Zhang}
\affiliation{%
  \institution{Hong Kong Polytechnic University}
  \city{Hong Kong}
  \country{China}
}

\author{Zichang Guo}
\affiliation{%
  \institution{Hong Kong Polytechnic University}
  \city{Hong Kong}
  \country{China}
}

\author{Hanlin Gu}
\affiliation{%
  \institution{AI Group, WeBank}
  \city{Shenzhen}
  \country{China}
}

\author{Di Jiang}
\authornote{Corresponding Author}
\affiliation{%
  \institution{Hong Kong Polytechnic University}
  \city{Hong Kong}
  \country{China}
}

\author{Qing Li}
\affiliation{%
  \institution{Hong Kong Polytechnic University}
  \city{Hong Kong}
  \country{China}
}

\renewcommand{\shortauthors}{Mengze Hong et al.}

\begin{abstract}
Customer service automation has seen growing demand within digital transformation. Existing approaches either rely on modular system designs with extensive agent orchestration or employ over-simplified instruction schemas, providing limited guidance and poor generalizability. This paper introduces an orchestration-free framework using Task-Oriented Flowcharts (TOFs) to enable end-to-end automation without manual intervention. We first define the components and evaluation metrics for TOFs, then formalize a cost-efficient flowchart construction algorithm to abstract procedural knowledge from service dialogues. We emphasize local deployment of small language models and propose decentralized distillation with flowcharts to mitigate data scarcity and privacy issues in model training. Extensive experiments validate the effectiveness in various service tasks, with superior quantitative and application performance compared to strong baselines and market products. By releasing a web-based system demonstration with case studies, we aim to promote streamlined creation of future service automation.
\end{abstract}



\maketitle

\section{Introduction}  

\begin{figure}
    \includegraphics[width=0.9\linewidth,left]{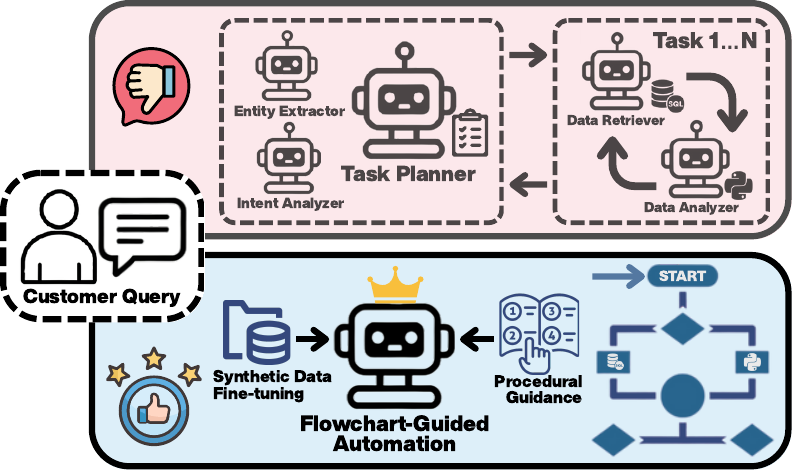}
    \vspace{-1.8em}
    \caption{Comparison of agent orchestration-based (top) and flowchart-guided (bottom) system for service automation.}
    \label{fig:orchestration}
    \vspace{-1em}
\end{figure}

Customer service plays a central role in building trust and sustaining growth in modern enterprise \cite{hong-etal-2025-augmenting, 10.1145/3701716.3717863}. Online platforms, such as e-commerce websites, social media, messaging applications (e.g., WhatsApp), and in-app service chatbots, have emerged as the primary channels for customers to submit inquiries and for businesses to manage client relationships, progressively replacing traditional telephone and email interactions \cite{walsh2000internet, xu2017new, wei2025asr}. This intensifies challenges for human-based service operations, which face higher labor costs, significant staff turnover due to increased workload pressures, and greater training demands, necessitating efficient and scalable AI-driven automation solutions \cite{cui2017superagent, shen2018does}.

Unlike socially engaging chit-chat dialogue systems \cite{choudhary-kawahara-2022-grounding}, customer service automation faces rigorous expectations. Inbound service queries demand precise problem resolution, while outbound communication aims to promote successful sales and gather customer information \cite{10.1145/3701716.3717545}. This is amplified in web-based applications concerning daily activities (e.g., online shopping, banking, food delivery), where the high accessibility and expectation for instant solutions drives the transformation of service automation from front-desk receptionist roles, often leading to the frustrating ``\textit{Transfer to human}'' option, to highly proficient and proactive problem-solvers that prioritize autonomous task completion \cite{xia2024digital}. This shift draws considerable attention from the web service and smart enterprise communities \cite{yom2018customer, liu2025multimodal, 10.1145/3442442.3451377, 10.1145/3366424.3382675} that advocate tailored solutions to improve service quality while reducing the operational costs associated with system creation and maintenance. 

Recent advancements in large language models (LLMs) have enabled task-oriented dialogue (TOD) systems with end-to-end dialogue management \cite{qin-etal-2023-end}. However, these systems either rely on overly simplified task schemas \cite{xu-etal-2024-rethinking}, resulting in insufficient task-oriented guidance, or require manual agent creation with heavy agent orchestration to combine modular components (see Figure \ref{fig:orchestration}) \cite{dong-etal-2025-protod}, leading to limited generalizability. Moreover, existing research emphasizes empirical performance over practical applications, resulting in significant challenges for industrial deployment. Large proprietary models (e.g., GPT) excel at dialogue tasks, but sensitive user data restricts external API access and requires local deployment \cite{liu2025towards}. Small language models (SLMs), on the other hand, support lightweight deployment but require targeted fine-tuning for service applications, which is often constrained by scarce and private dialogue training data \cite{fan2025fedmkt} and risks of model memorization for sensitive information \cite{carlini2023quantifying, ozdayi-etal-2023-controlling}.

To address these challenges, we propose a novel \textbf{flowchart-guided customer service automation} framework. The task-oriented flowchart structure is introduced to streamline complex dialogue interactions \cite{sokovic2009basic, zhang2024flowce}, enabling orchestration-free zero-shot task coordination and synthetic dialogue data generation for training small language models. The contributions of this paper include:

\begin{enumerate}
\item \textbf{Procedural Knowledge Representation}: We propose the task-oriented flowchart (TOF) structure to augment service automation with systematic procedural guidance, presenting rigorous component definitions and evaluation metrics to assess the comprehensiveness of task coverage.

\item \textbf{Dialogue-to-Flowchart Optimization}: We propose a two-step approach to construct TOFs from service dialogue: selecting representative samples and creating flowcharts via an iterative algorithm, emphasizing cost efficiency.

\item \textbf{Flowchart-Guided Service Automation}: With the constructed TOFs, we introduce a flowchart-guided framework for orchestration-free service automation, featuring a prompting strategy and a decentralized distillation approach for training SLMs without direct data exposure.

\item \textbf{System Validation}: Through rigorous experiments, we validate the effectiveness of flowcharts in enhancing service success, supported by empirical evidence from benchmark evaluations and qualitative assessments with human experts in practical system deployments. A web-based chatbot demonstration with case studies is released, serving as a foundation for future studies to build upon this work\footnote{\href{https://github.com/mengze-hong/Flowchart-Guided-Service-Automation}{https://github.com/Flowchart-Guided-Service-Automation}}.
\end{enumerate}

\begin{figure*}
    \centering
    \includegraphics[width=0.86\linewidth]{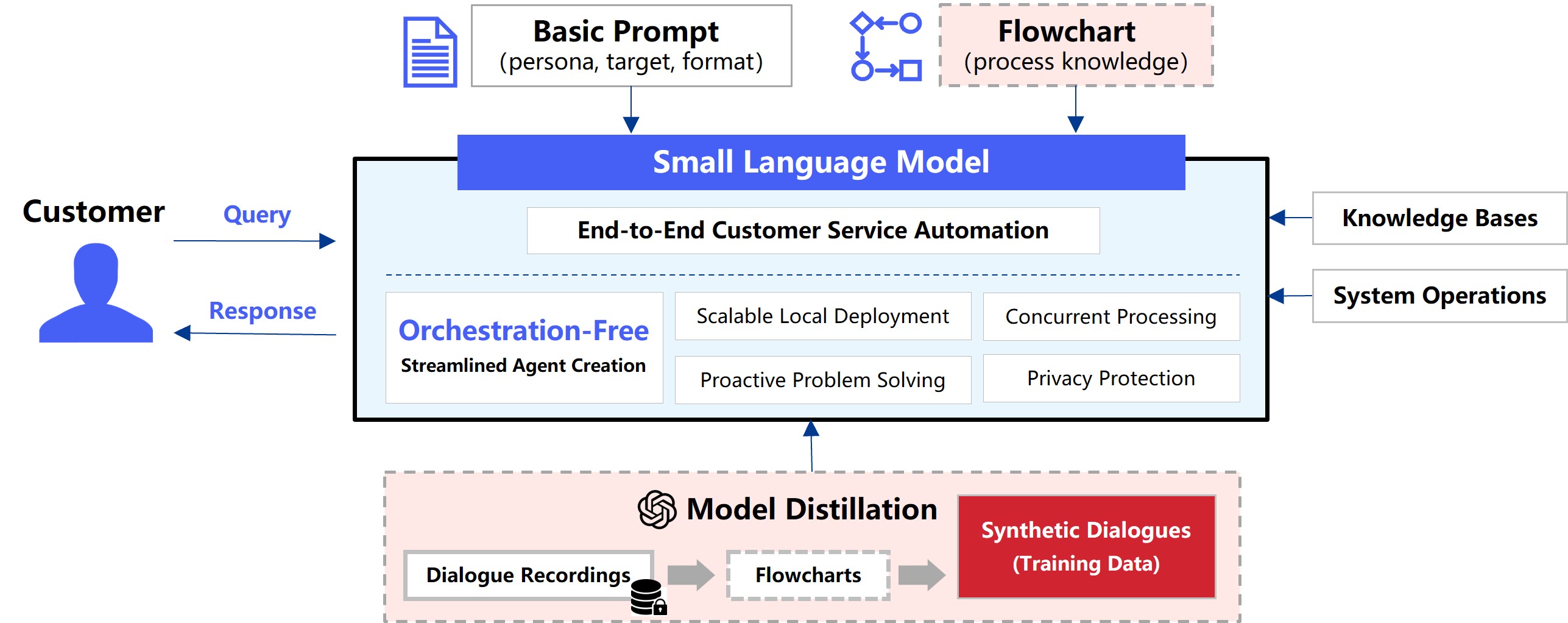}
    \caption{System overview: orchestration-free and privacy-preserving customer service automation driven by 1) flowchart-augmented task coordination and 2) flowchart-guided synthetic data generation for model distillation.}
    \label{fig:system overview}
\end{figure*}

\section{Related Work}

\subsection{Customer Service Automation}
The evolution of digital services has transformed the way customers interact with businesses, with platforms such as e-commerce and e-banking increasingly embedding support features directly into their websites and mobile applications \cite{cao2018post}. Early customer service automation was dominated by rule-based systems such as Interactive Voice Response (IVR), which guided customers through fixed menus to answer common questions, or route calls to human agents \cite{corkrey2002interactive}. These systems struggled with non-trivial queries, required frequent human intervention, and faced criticism for poor accessibility and limited adaptability. Later, machine learning-based dialogue models introduced intent recognition and slot-filling capabilities \cite{gasic-etal-2013-pomdp, wen-etal-2017-network}, enabling more dynamic and goal-oriented interactions. However, these systems relied on carefully engineered datasets and rigid ontologies for model training, offering limited scalability.

The emergence of large language models (LLMs) has transformed service automation with advanced natural language understanding and generation capabilities \cite{naveed2025comprehensive, an2025unictokens}, demonstrating successful applications in domains such as healthcare \cite{wei-etal-2018-task, valizadeh2022ai} and e-commerce \cite{yan2017building}. Unlike LLM-driven QA chatbots that rely on retrieval-based mechanisms to generate simple responses \cite{ral2m, hong-etal-2025-augmenting}, task-oriented dialogue (TOD) systems help resolve problems and are crucial for realizing true service automation. ProTOD \cite{dong-etal-2025-protod} employs a pipeline architecture with three LLM-powered modular components (State Tracker, Knowledge Retriever, and Proactive Policy Planner) to enhance goal completion and proactivity. In contrast, SimpleTOD \cite{NEURIPS2020_e9462095} demonstrates that a single language model excels in dialogue tasks through transfer learning, while AutoTOD \cite{xu-etal-2024-rethinking} validates the effectiveness of an LLM's instruction-following capability with semi-structured task schemas to autonomously determine system actions, enabling seamless multi-domain task completion through end-to-end dialogue management. Additionally, retrieval-augmented generation (RAG) improves response generation with external knowledge \cite{xu2024retrieval}, while model fine-tuning enhances dialogue consistency \cite{hong-etal-2025-dialogue}.

Despite recent advancements in system designs, modular architectures frequently suffer from high latency, unstable performance, and significant operational costs due to agent creation and orchestration \cite{clarke-etal-2022-one, zhang-etal-2025-unifying}. End‑to‑end systems, while reducing design overhead, lack interpretability and struggle with complex task coordinations \cite{xu-etal-2024-rethinking}. These limitations underscore the need for a cost‑efficient solution capable of reliably navigating complex service scenarios.

\subsection{Flowchart Understanding and Application}

Flowcharts, represented as directed graphs of sequential or conditional processes, effectively clarify complex workflows and decision paths \cite{sokovic2009basic} and are especially valuable for \textbf{capturing and representing business logic} \cite{gronroos2011service}. Traditionally, business processes are documented as Standard Operating Procedures (SOPs) that define a sequence of actions and best practices, while business requirements specify the expectations or performance targets to be met \cite{dumas2018fundamentals}. A flowchart integrates these two forms of knowledge into a cohesive workflow \cite{damij2007business}, ensuring that adherence to the process ultimately leads to the fulfillment of the defined business requirements.

Recent studies have investigated the capabilities of LLMs and multimodal LLMs (MLLMs) in flowchart comprehension, demonstrating strong proficiency in interpreting workflow processes \cite{pan2024flowlearn, suri2025followflowfinegrainedflowchart}. However, the application of flowcharts remains limited to decision-making and simple Q\&A-based systems \cite{yamanaka-etal-2025-flowchart, raghu-etal-2021-end}, overlooking their potential in procedural knowledge abstraction and system design as a framework for guiding service automation. Notably, unlike open-domain chatbot applications, customer service operates in a closed domain, mostly addressing company-related queries, where pre-defined procedural knowledge can be effectively leveraged to guide dialogue flow and improve task completion.

\vspace{1em}
\noindent \textbf{Highlights.} Unlike prior studies, this paper presents an end-to-end, orchestration-free customer service automation using novel task-oriented flowcharts. It delivers procedural guidance for task coordination in comparison to unstructured task schemas and supports the training of lightweight SLMs for local deployment, achieving streamlined system development with strong performance.

\section{Task-Oriented Flowchart}
In customer service, conversations typically focus on a single domain and often involve multiple interconnected goals, such as booking a restaurant, arranging nearby accommodation, or modifying prior details with additional constraints. These interactions can be conceptualized as structured workflows driven by customer intentions, with each goal corresponding to a specific task executed by the service agent through a coherent sequence of dialogue actions, closely aligning with the design of task-oriented systems. The key to guiding customer service automation lies in delivering clear procedural knowledge of these intents and actions \cite{CHEN20119}. To address this, we propose a novel \textbf{\textit{Task-Oriented Flowchart}} (TOF) structure that models dialogues as procedural sequences to guide task completion.

\subsection{Definition and Composition}

A flowchart is a directed graph \( G = (V, E) \), where \( V \) denotes nodes representing process states or actions, and \( E \subseteq V \times V \) represents transitions. Unlike traditional acyclic flowcharts with rigid sequential flows, the proposed TOF permits loops for error handling and re-prompting. We propose five TOD-specific node types tailored for systematic flowchart design for the customer service application:

\begin{itemize}  
    \item \textbf{Start Node}: Initiates the dialogue by embedding the scenario and defining the agent's role (e.g., ``Begin account inquiry''), aligning the service agent with the customer's intention.
    \item \textbf{Action / Decision Node}: Executes system‑level operations and manages conditional logic, such as input validation (e.g., ``Is account number valid?'') and database queries.  
    \item \textbf{Output Node}: Provides responses, including information display (e.g., ``Display balance'') and confirmations.  
    \item \textbf{Reflection Node}: Evaluates the dialogue state to adapt strategies and recover from errors.
    \item \textbf{End Node}: Finalizes the task (e.g., ``Complete inquiry'') and marks goal completion based on user feedback.  
\end{itemize}  

\noindent Edges, labeled with conditions (e.g., ``valid input'', ``user confirmed''), connect nodes. The non-acyclic design enables loops from output or reflection nodes to action nodes and from reflection or end nodes to the start node, facilitating iterative refinement of requirements or initiation of new task queries. The flowchart is represented in \textit{Mermaid Markdown}\footnote{https://mermaid.live/edit} for seamless textual input and flexible diagram conversion (see examples in \textbf{Appendix \ref{sec:flowchart demo}}). This intent‑driven structure marks a pioneering advance, systematically mapping dialogue intents and actions to procedural abstraction.

\subsection{Flowchart Evaluation}
\label{sec:flowchart evaluation}

With the goal of guiding service automation, it is essential to assess the TOF's ability to capture diverse customer intentions and agent actions. Inspired by the coverage metrics widely employed in formal verification \cite{chockler2003coverage} and requirement-based testing \cite{whalen2006coverage}, we formulate two complementary metrics: \textit{Utterance Matching Ratio} and \textit{Complete Path Coverage}, evaluating the flowchart's effectiveness in modeling dialogue trajectories and capturing utterance-level semantics.

\subsubsection*{Utterance Matching Ratio (UMR)} Given a flowchart $F = (V, E)$, we formulate utterance-node mapping as a classification problem \(\psi(u_i) \in V \cup \{\emptyset\}\), where each utterance \(u_i\) from a dialogue \(D = \{u_1, \dots, u_n\}\) is assigned to exactly one node in $F$ or to \(\emptyset\), representing the ``no‑match'' case. The Utterance Matching Ratio for a single dialogue $D$ is then given by:  
\[
\text{UMR}(D, F) = \frac{\lvert \{ u_i \in D \mid \psi(u_i) \neq \emptyset \} \rvert}{n},
\]  
where \(n\) is the utterance count in \(D\). The dataset‐level evaluation is reported as:
\[
\text{UMR}_{\text{avg}}(F, \mathbf{D}) = \frac{1}{|\mathbf{D}|} \sum_{D_i \in \mathbf{D}} \text{UMR}(D_i, F).
\]  
This captures the proportion of utterances successfully matched to flowchart nodes, reflecting semantic coverage and granularity.

\subsubsection*{Complete Path Coverage (CPC)}  
We define a complete path \(P = (v_1, v_2, \dots, v_m)\) from flowchart $F$ as an ordered sequence of nodes such that \((v_j, v_{j+1}) \in E\) for all \(j \in \{1, \dots, m-1\}\), with \(v_1\) denoting the designated start node and \(v_m\) the end node. We model dialogue-path coverage as a structured mapping problem \(\phi: D \to P_D\), wherein each utterance in \(D_i\) is first mapped to a node, and the resulting node sequence is evaluated for completeness. A dialogue is said to be covered by a complete path if and only if $P_D$ contains both \(v_1\) and \(v_m\), with the index of \(v_1\) preceding that of \(v_m\), and no unmatched nodes (\(\emptyset\)) occur in between. The Complete Path Coverage metric is thus defined as:  
\[
\text{CPC}(F, \mathbf{D}) = \frac{\lvert \{ D_i \in \mathbf{D} \mid \phi(D_i) \ \text{yields a complete path in} \ F \} \rvert}{|\mathbf{D}|},
\]  
\noindent capturing the proportion of dialogues that can be fully supported by a valid trajectory from start to end within the flowchart. These metrics can be implemented via semantic classification using transformer-based encoders (e.g., BERT) fine-tuned for node discrimination, or lightweight generative LMs with in-context learning, where node descriptions guide classification without task-specific training \cite{zhang2025pushing}.

\section{Methodology}

Manual creation of service agents is time-consuming and requires complex orchestration. Leveraging accumulated human-human service dialogues, we first propose a two-step approach for constructing TOFs in Section \ref{sec:flowchart construction}. These flowcharts capture business logic with procedural knowledge, facilitating streamlined agent creation, as discussed in the subsequent Section \ref{sec:flowchart application}.

\subsection{Flowchart Construction with Dialogue Data}
\label{sec:flowchart construction}

\begin{algorithm}[!t]
\small
\caption{LP-Relaxation and Rounding for WDIC}
\label{alg:wdic-lp}
\begin{algorithmic}[1]
\REQUIRE Intent universe $\mathcal{I}$, dataset $D$ with subsets $S_i$, costs $c_i$
\ENSURE Subset $\mathbf{D'} \subseteq D$ covering $\mathcal{I}$
\STATE Set $\alpha \gets \lceil \ln |\mathcal{I}| \rceil$
\STATE Initialize $\mathbf{D'} \gets \emptyset$, $\mathcal{U} \gets \mathcal{I}$
\FOR{$i = 1$ to $k$}
    \IF{$\text{Uniform}[0,1] \leq \min(1, \alpha \cdot x_i^*)$}
        \STATE $\mathbf{D'} \gets \mathbf{D'} \cup \{D_i\}$, $\mathcal{U} \gets \mathcal{U} \setminus S_i$
    \ENDIF
\ENDFOR
\WHILE{$\mathcal{U} \neq \emptyset$}
    \STATE Select $D_j \notin \mathbf{D'}$ maximizing $|S_j \cap \mathcal{U}| / c_j$
    \STATE $\mathbf{D'} \gets \mathbf{D'} \cup \{D_j\}$, $\mathcal{U} \gets \mathcal{U} \setminus S_j$
\ENDWHILE
\STATE \textbf{Output:} $\mathbf{D'}$
\end{algorithmic}
\end{algorithm}

\subsubsection{Dialogue Selection}
\label{sec:WDIC}

Given a set of dialogues $\mathbf{D} = \{D_1, \dots, D_k\}$, each dialogue $D_i$ covers a set of intents $S_i \subseteq \mathcal{I}$ expressed by the utterances (e.g., book-hotel, inquire-price), where $\mathcal{I}$ denotes the set of all distinct intents in $\mathbf{D}$. We formulate the \textbf{Weighted Dialogue Intent Coverage (WDIC) Problem} as selecting a minimal subset of dialogues that collectively cover all intents, enabling cost-efficient TOF construction. Each dialogue has an associated cost $c_i$, defined as either a unit cost of $1$ (unweighted) or the number of utterances $|D_i|$ (weighted). The objective is to choose a subset of dialogues covering every intent while minimizing total cost. Formally, the optimization problem is defined as:

\begin{equation}
    \text{minimize } \sum_{D_i \in \mathbf{D'}} c_i \quad \text{subject to} \quad \bigcup_{D_i \in \mathbf{D'}} S_i = \mathcal{I}.
\end{equation}

We establish the NP-hardness of the WDIC problem through a reduction from the Set Cover problem \cite{cormen2022introduction}. Additionally, we show that the objective function is submodular, enabling a greedy algorithm with known approximation guarantees (see Appendix \ref{sec:proofs}). Since greedy solutions may be suboptimal, we transform the problem into an Integer Linear Program (ILP) to derive exact solutions:
$$
\begin{aligned}
\text{minimize} \quad & \sum_{i=1}^k c_i x_i \\
\text{subject to} \quad & \sum_{i : \iota \in S_i} x_i \geq 1, \quad \forall \iota \in \mathcal{I}, \\
& x_i \in \{0,1\} \quad \forall i \in \{1, \dots, k\},
\end{aligned}
$$
where $  \iota  $ represents an intent in $  \mathcal{I}  $, and $  x_i = 1  $ indicates the selection of dialogue $  D_i  $ with cost $  c_i  $. For large datasets ($  k \approx 10^4  $, $  |\mathcal{I}| \approx 10^3  $), exact ILP solutions are computationally expensive. We therefore adopt the approximation as described in Algorithm \ref{alg:wdic-lp}, relaxing binary variables to \(0 \leq x_i \leq 1\) and solving the resulting LP in polynomial time to obtain a fractional solution \(x_i^*\) \cite{lasserre2005polynomial}. Randomized rounding then includes each dialogue \(D_i\) with probability \(\min(1, \alpha \cdot x_i^*)\), where \(\alpha = \lceil \ln |\mathcal{I}| \rceil\), yielding an expected \(O(\log |\mathcal{I}|)\)-approximation, followed by a greedy refinement to ensure complete coverage. This approach matches the greedy bound while providing a more cost‑efficient solution in practice.

\subsubsection{Flowchart Construction}

Given the selected dialogue subset, we propose an iterative, intent-aware method to construct the TOFs. Each dialogue \(D_k\) is represented as an ordered sequence of \(n_k\) customer-agent utterance pairs \((u_c, u_a)\):
\[
\Pi(D_k) = \{(u_c^i, u_a^i)\}_{i=1}^{p_k}, \quad p_k = \lfloor n_k / 2 \rfloor.
\]
To handle dialogues spanning multiple queries across distinct service domains, we employ a locally deployed SLM (e.g., Qwen-8B) to drive three task-specific oracles: (i) \(\phi\), which extracts concise, verb-initial intent descriptors (e.g., ``inquire-price'') following \cite{hong2025dialin}; (ii) \(\psi\), which identifies the relevant query type or domain (e.g., restaurant, hotel); and (iii) \(\tau\), which assigns node types as specified by TOF. 

The algorithm initializes a global root node and processes each dialogue sequentially. For each utterance pair, the oracle \(\psi\) first determines the domain \(d_k\) to assign the pair to the appropriate subgraph \(S_{d_k}\). Subsequently, \(\phi\) extracts the intent descriptor \(I_i\), enabling the system to either \textbf{reuse an existing node with a matching intent} or \textbf{create a new node} with a type-specific structure determined by \(\tau\), adding directed edges to maintain dialogue flow. The resulting flowchart is serialized into Mermaid markdown. The algorithm operates with a linear time complexity of \(O\left(\sum_k p_k \cdot |\mathcal{M}|\right)\), where \(|\mathcal{M}|\) denotes the average latency of the SLM. As dialogue selection pre-eliminates repeated dialogues and reduces the sample size, this facilitates scalable knowledge abstraction and allows new business requirements to be easily integrated by adding new nodes.

\subsection{Flowchart-Guided Customer Service Automation Framework}
\label{sec:flowchart application}

Existing research has largely overlooked the role of flowcharts in practical applications. To address this gap, we introduce two approaches that use task-oriented flowcharts to guide and justify the development of orchestration-free customer service automation.

\subsubsection{Flowchart-Augmented Task Coordination}  
We first address the lack of procedural guidance in unstructured task instruction by proposing a composite flowchart-augmented prompting strategy:
\[
\mathcal{P}_{\text{f}} = 
\underbrace{\mathcal{P}_{\text{nl}}}_{\text{global task description}}
\; \cup \;
\underbrace{\mathcal{P}_{\text{struct}}\left(\{F_i, W_i\}_{i=1}^M\right)}_{\text{flowchart--schema knowledge}},
\]  
where $\mathcal{P}_{\text{nl}}$ is the basic prompt specifying the overall operational scope and dialogue style, and $\{F_i, W_i\}_{i=1}^M$ denotes $M$ TOFs $F_i = (V_i, E_i)$, each paired with a node-schema mapping $W_i: V_i \to \mathcal{C}_i$ that specifies required parameters $\mathcal{C}_i(v)$ for the ``action'' node (e.g., schema to perform SQL query). To further enhance system performance, we embed explicit procedure tracking in $\mathcal{P}_{\text{f}}$, requiring the agent to identify the active node $v_t$ in the relevant flowchart at each dialogue turn,  
\[
f_i : u_t \mapsto v_t, \quad v_t \in V_i,
\]  
where $u_t$ is the user's utterance at turn $t$. This tracking mechanism replaces conventional state management, maintaining an exact reference to the agent's position within the task flow as part of the natural generation process, thereby enhancing interpretability and facilitating precise error tracking in the end-to-end system design.

\subsubsection{Synthetic Data Generation for Model Training}

\label{sec:distillation}

\begin{figure}
    \centering
    \includegraphics[width=1\linewidth]{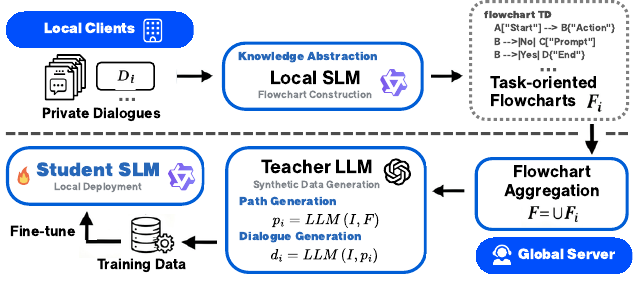}
    \vspace{-1.9em}
    \caption{Decentralized flowchart distillation framework.}
    \label{fig:distillation}
\end{figure}

Training specialized dialogue models with small language models (SLMs) facilitates lightweight local deployment, enhancing service quality while ensuring compliance with data privacy regulations (i.e., user data cannot be exposed to third-party APIs). However, raw dialogue data are extremely noisy, and real service dialogues are confidential and contain sensitive user information. Moreover, a single service scenario cannot adequately capture the diversity of user queries and agent response strategies, necessitating data from multiple service departments or companies. While federated learning offers a solution by decentralized model training \cite{10.1145/3298981}, the risk of training data leakage persists, as LLMs are vulnerable to memorization attacks that may extract sensitive data from model outputs \cite{meng-etal-2025-r, carlini2023quantifying, ozdayi-etal-2023-controlling}.

\begin{algorithm}[t]
\small
\caption{Decentralized Flowchart Construction}
\label{alg:flowchart_distillation}
\begin{algorithmic}[1]
\REQUIRE Local datasets \( \mathbf{D} = \{ D_1, \dots, D_N \} \), local SLMs, classifier \( \mathcal{M} \)
\ENSURE Global flowchart \( F_{\text{global}} \)
\vspace{0.5em}
\FOR{$i = 1$ to $N$}
    \STATE \( F_i \gets \text{SLM}(D_i) \)
\ENDFOR
\STATE \( \mathcal{U} \gets \bigcup_{i=1}^N V(F_i) \)
\STATE \( \{ C_1, \dots, C_K \} \gets \text{Cluster}(\mathcal{U}) \)
\FOR{$k = 1$ to $K$}
    \STATE \( g_k \gets \mathcal{M}(C_k) \)
    \IF{$g_k = 1$}
        \STATE \( v_k \gets \text{Merge}(C_k) \)
        \STATE \( \text{Pre}(v_k) \gets \bigcup_{u \in C_k} \text{Pre}(u) \)
        \STATE \( \text{Post}(v_k) \gets \bigcup_{u \in C_k} \text{Post}(u) \)
    \ENDIF
\ENDFOR
\STATE \( V_{\text{global}} \gets \bigcup_{k=1}^K \left[ \mathbb{I}(g_k = 1) \cdot \{ v_k \} \ \cup \ \mathbb{I}(g_k = 0) \cdot C_k \right] \)
\STATE \( E_{\text{global}} \gets \{ (u,v) \mid v \in V_{\text{global}},\ u \in \text{Pre}(v) \cup \text{Post}(v) \} \)
\RETURN $(V_{\text{global}}, E_{\text{global}})$
\end{algorithmic}
\end{algorithm}

To tackle these practical challenges, we introduce a \textbf{decentralized flowchart distillation framework}, designed to facilitate privacy‑centric, scalable training of dialogue models. Privacy is preserved by locally transforming dialogue datasets $\mathbf{D} = \{D_i\}_{i=1}^N$ into TOFs with the proposed method in Section~\ref{sec:flowchart construction}, using locally deployed SLM. The global flowchart $F_{\text{global}}$ is constructed by aggregating local flowcharts following Algorithm~\ref{alg:flowchart_distillation}. Inspired by \cite{hong2025dialin}, nodes from $\{F_i\}_{i=1}^M$ are pooled and clustered, with each cluster assessed for semantic coherence by the LLM classifier $\mathcal{M}$. Coherent clusters are merged into a single representative node with edges reassigned to reduce redundancy, whereas incoherent clusters retain their original nodes and connections. The resulting aggregated flowchart contains only high‑level procedural logic, without disclosing any sensitive or personally identifiable information. Differential privacy \cite{zhao2022survey}, though a common safeguard, has limited practical value in this context. Its application introduces noise into synthetic data generation, potentially compromising flow coherence and utility, while providing minimal additional protection since the flowchart itself poses no privacy concerns.

From $F_{\text{global}}$, the framework selects feasible dialogue paths:
\[
p_j \sim \pi(\mathcal{P}(F_{\text{global}})),
\]
where $\mathcal{P}(F_{\text{global}})$ denotes the set of valid paths through the aggregated flowchart, and $\pi$ is a sampling policy that can be optimized to balance coverage of common flows with exploration of rare branches, promoting diversity in the generated data and reducing bias toward frequent workflows. For simplicity, random sampling at diverse dialogue lengths is used. For each selected path $p_j$, a proprietary LLM (e.g., GPT-4) generates synthetic dialogues aligned with the node sequence in $p_j$:
\[
\mathbf{D_{\text{synth}}} = \{ (\mathcal{M}(p_j), \mid p_j \in \mathcal{P}(F_{\text{global}}) \},
\]
producing high-quality training samples that are procedurally valid, lexically diverse, and consistent with the service workflow. 

In the fine‑tuning phase, we adopt a \emph{structure‑aware} approach that embeds $F_{\text{global}}$ alongside the dialogue context, forming training samples $(\{F_{\text{global}}, x_k\}, y_k)$. This provides explicit procedural priors, reinforcing the model's ability to understand and follow the flowchart. The training objective is expressed as:
\[
\mathcal{L}_{\text{FT}}
= \mathbb{E}_{(F_{\text{global}}, x_k, y_k)}
\left[
\mathcal{L}_{\text{CE}}\!\left( p_{\theta}(\cdot \mid x_k, F_{\text{global}}),\, y_k \right)
\right],
\]
where $p_{\theta}$ is the student model parameterized by $\theta$. This distillation approach ensures access to sufficient and diverse training data while keeping local datasets private and preventing raw data exposure during training, thus providing two layers of privacy protection.

\begin{table*}[!t]
\renewcommand{\arraystretch}{1.2} 
\centering
\caption{Coverage evaluation of flowchart construction methods on the MultiWoZ 2.0 dataset.}
\vspace{-0.5em}
\label{tab:coverage_results}
\begin{tabular}{|l|cc|cc|cc|cc|cc|}
\hline
\multirow{2}{*}{\textbf{Method}} & \multicolumn{2}{c|}{\textbf{Attraction}} & \multicolumn{2}{c|}{\textbf{Hotel}} & \multicolumn{2}{c|}{\textbf{Restaurant}} & \multicolumn{2}{c|}{\textbf{Taxi}} & \multicolumn{2}{c|}{\textbf{Train}} \\
\cline{2-11}
& \textbf{CPC} & \textbf{UMR} & \textbf{CPC} & \textbf{UMR} & \textbf{CPC} & \textbf{UMR} & \textbf{CPC} & \textbf{UMR} & \textbf{CPC} & \textbf{UMR} \\
\hline
Human-Annotated & 0.86 & 0.8176 & 0.76 & 0.7966 & 0.80 & 0.8071 & \textbf{0.92} & 0.6728 & 0.84 & 0.7183 \\
Abstraction-Based & 0.80 & 0.7958 & 0.64 & 0.8103 & 0.74 & 0.8044 & 0.90 & 0.6843 & 0.74 & 0.7987 \\
\textbf{Iterative Construction} & \textbf{0.92} & \textbf{0.9436} & \textbf{0.80} & \textbf{0.8209} & \textbf{0.90} & \textbf{0.9087} & 0.88 & \textbf{0.8221} & \textbf{0.88} & \textbf{0.8277} \\
\hline
\end{tabular}
\end{table*}

\begin{table}[!t]
\centering
\caption{Comparison of dialogue selection optimizations. The best performing method (i.e., lowest cost) is marked in \textcolorblue{blue}.}
\vspace{-0.5em}
\renewcommand{\arraystretch}{1}
\resizebox{\columnwidth}{!}{
\begin{tabular}{lcc|cc}
\toprule
 & \multicolumn{2}{c}{\textbf{MultiWOZ}} & \multicolumn{2}{c}{\textbf{SGD}} \\
\cmidrule(lr){2-3} \cmidrule(lr){4-5}
\textbf{Method} & \textbf{\# Dialog} & \textbf{\# Utterances} & \textbf{\# Dialog} & \textbf{\# Utterances} \\
\midrule
\multicolumn{5}{c}{\textbf{Unweighted DIC}} \\
\midrule
Greedy & 57 & 1012 & 23 & 318 \\
LP-Rounding & 67 & 1140 & 37 & 596 \\
ILP & \colorbox{MyBlue}{\textbf{43}} & 766 & 21 & 344 \\
\midrule
\multicolumn{5}{c}{\textbf{Weighted DIC}} \\
\midrule
Greedy & 74 & 756 & 32 & 330 \\
LP-Rounding & 68 & 828 & 42 & 580 \\
ILP & 51 & \colorbox{MyBlue}{\textbf{620}} & \colorbox{MyBlue}{\textbf{21}} & \colorbox{MyBlue}{\textbf{260}} \\
\bottomrule
\end{tabular}
}
\label{tab:wdic-comparison}
\vspace{-1em}
\end{table}

\section{Experiments}

We first evaluate the dialogue selection and flowchart construction methods, followed by quantitative benchmarks against various TOD baselines in task completion performance. Given the limited coverage of outbound service scenarios, we conduct a qualitative study by deploying the system in a real service setting and comparing it with market products for outbound communication. Finally, we present a web-based chatbot demonstration with case studies to illustrate how flowcharts enhance customer interactions.

\subsection{Dialogue-to-Flowchart Evaluation}

We evaluate the proposed flowchart construction methods on two dialogue datasets: \textbf{MultiWOZ 2.0} \cite{budzianowski-etal-2018-multiwoz}, containing 9,906 dialogues with 264 intents spanning five domains, and \textbf{SGD} \cite{rastogi2020towards}, covering 768 dialogues with 256 intents.

\subsubsection{Dialogue Selection}
The statistics of the selected dialogue subsets in Table~\ref{tab:wdic-comparison} show that all proposed approaches achieve complete intent coverage with compact subsets. The ILP method outperforms the greedy algorithm, consistently minimizing both dialogue and utterance counts, reducing the sample size from 9,906 to 43 dialogues for the MultiWoZ dataset. The weighted DIC, using utterance count as the selection cost, outperforms the unweighted variant. Specifically, ILP selects 21 dialogues comprising 260 utterances, compared to 21 dialogues with 344 utterances in unweighted DIC, demonstrating \textbf{enhanced efficiency through finer cost granularity}. Although LP-Rounding yields suboptimal subsets, it runs 41.3\% faster than ILP and 66.2\% faster than the greedy algorithm, making it a suitable alternative for large-scale industrial deployment.

\subsubsection{Flowchart Construction}

We construct flowcharts from ILP-selected dialogues using the proposed iterative construction approach, and compare them against a human-annotated reference and an abstraction-based method, which first summarizes dialogues into task requirement documents, then applies few-shot prompting to mimic human annotation\footnote{Implementation details and annotator profiles are provided in Appendix~\ref{sec:implementation}.}. We evaluate these flowcharts on 50 randomly sampled dialogues per domain from the MultiWoZ dataset and report results using the CPC and UMR metrics, as defined in Section~\ref{sec:flowchart evaluation}, with CPC relaxed to require at least one complete path to accommodate dialogues spanning multiple tasks.

As shown in Table~\ref{tab:coverage_results}, all methods achieve satisfactory coverage of intents and dialogue paths. Interestingly, the iterative construction approach outperforms human annotation, which can be attributed to its incremental refinement in capturing fine-grained business logics through the step-by-step construction process. This demonstrates that comprehensive TOFs can be automatically constructed from service dialogues, highlighting the potential for \textbf{fully data‑driven agent creation and digital transformation}.

\begin{table}[!t]
\centering
\caption{Goal completion evaluation on MultiWoz 2.0 and SGD. The best performing method is marked in \textcolorblue{blue}.}
\vspace{-0.5em}

\resizebox{\columnwidth}{!}{
\begin{tabular}{lccccc}
\toprule
& \multicolumn{3}{c}{\textbf{MultiWOZ}} & \multicolumn{2}{c}{\textbf{SGD}} \\
\cmidrule(lr){2-4} \cmidrule(lr){5-6}
\textbf{Method} & \textbf{Inform} & \textbf{Success} & \textbf{Book} & \textbf{Inform} & \textbf{Success}\\
\midrule
\multicolumn{6}{l}{\textbf{\textit{Training-based Approaches}}} \\
SimpleTOD & 84.4 & 70.1 & - & 12.7 & 9.8\\
UBAR & 83.4 & 70.3 & - & - & -\\
GALAXY & 85.4 & 75.7 & - & - & -\\
Mars & 88.9 & 78.0 & - & - & -\\
\midrule
\multicolumn{6}{l}{\textbf{\textit{Prompting-based Approaches}}} \\
SGP-TOD (GPT3.5) & 83.9 & 69.9 & - & - & -\\
AutoTOD (GPT-4) & 87.2 & 82.8 & 81.4 & 45.1 & 23.0\\
ProTOD (GPT-3.5) &  \colorbox{MyBlue}{\textbf{91.7}} & 83.3 & 87.0 &  \colorbox{MyBlue}{\textbf{50.4}} & 24.9\\
\midrule
\multicolumn{6}{l}{\textbf{\textit{Flowchart-Guided Approaches}}} \\
Ours (GPT-3.5) & 88.3 & 84.6 & \colorbox{MyBlue}{\textbf{91.4}} & 46.3 &  \colorbox{MyBlue}{\textbf{26.8}}\\
Ours (LLaMA-8B) & 84.7 & \colorbox{MyBlue}{\textbf{85.9}} & 90.7 & 42.5 & 25.4\\
\bottomrule
\end{tabular}}
\label{tab:quantitative results}
\vspace{-0.5em}
\end{table}

\subsection{Task Completion Evaluation}

With the constructed flowcharts, we compare our methods against two categories of baselines: (i) training‑based TOD models, and (ii) prompting‑based approaches. Evaluations are conducted on the MultiWoZ 2.0 and SGD datasets, following the setup in \cite{xu-etal-2024-rethinking}, where \textit{Inform} measures correct entity provision, \textit{Success} indicates whether all requested attributes are satisfied, and \textit{Book} represents the success rate of bookings or reservations. We apply the flowchart‑augmented prompt to the GPT‑3.5 model. The distillation strategy is implemented by generating 500 dialogue samples per domain using GPT‑4, followed by fine‑tuning a LLaMA‑3‑8B‑Instruct model, which is subsequently evaluated using the same prompting strategy.

As shown in Table~\ref{tab:quantitative results}, our flowchart-guided approaches achieved strong overall performance with notable improvements over baseline approaches. On MultiWOZ~2.0, our prompt-based approach with GPT‑3.5 achieves a Success rate with relative gains of approximately 2.17\% over AutoTOD and 1.56\% over ProTOD. This improvement is achieved despite AutoTOD utilising GPT‑4, highlighting that the integration of flowcharts can compensate for differences in model scale through embedded business logic. Furthermore, the fine-tuned LLaMA SLM confirmed that flowchart-guided distillation enables even smaller open-source models to match or surpass larger proprietary models in task completion. This can be attributed to the nature of customer service, where agents are mostly responsible for domain-specific queries and often prefer to ignore the out-of-domain requests, which \textbf{benefit the most from model fine-tuning with domain-specific data}. Similar results are observed in the SGD dataset, with both proposed methods achieving higher success rates than various baselines.

The Inform metric highlights a limitation of our method: it underperforms compared to ProTOD, which employs a modular orchestration-based design integrating a dialogue state tracker, exploratory knowledge retriever, and policy planner. This approach provides more detailed guidance in defining information flow. In contrast, our method prioritizes high-level business logic abstractions that \textbf{emphasize overall goal completion rather than fine-grained entity matching}, leading to the slightly lower Inform score. However, considering that flowcharts are easier to construct from dialogues, while agent orchestration incurs higher computational costs and limited scalability, we argue that the compromise in the Inform metric does not diminish the overall practical value of our system. The superior end-to-end task completion and multi-step objective handling capabilities demonstrated by the proposed system position procedural knowledge augmentation through TOFs as a compelling candidate for real-world system deployment.

\begin{figure}[!t]
    \centering
    \includegraphics[width=0.8\linewidth]{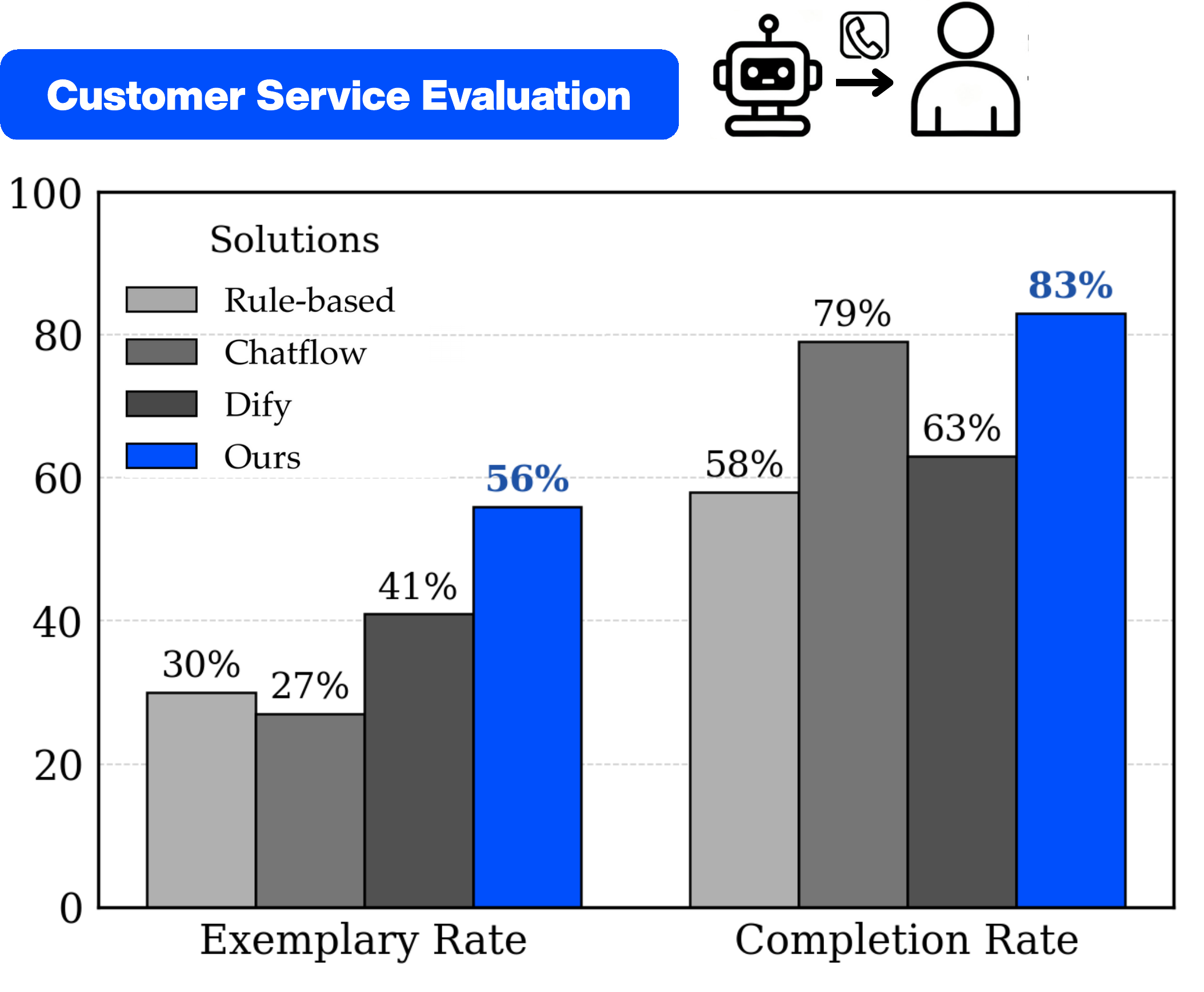}
    \vspace{-1.3em}
    \caption{Comparison of Exemplary Rate and Completion Rate for four customer service outbound automations.}
    \vspace{-1em}
    \label{fig:qualitative result}
\end{figure}

\begin{figure}[!t]
    \centering
    \includegraphics[width=0.8\linewidth]{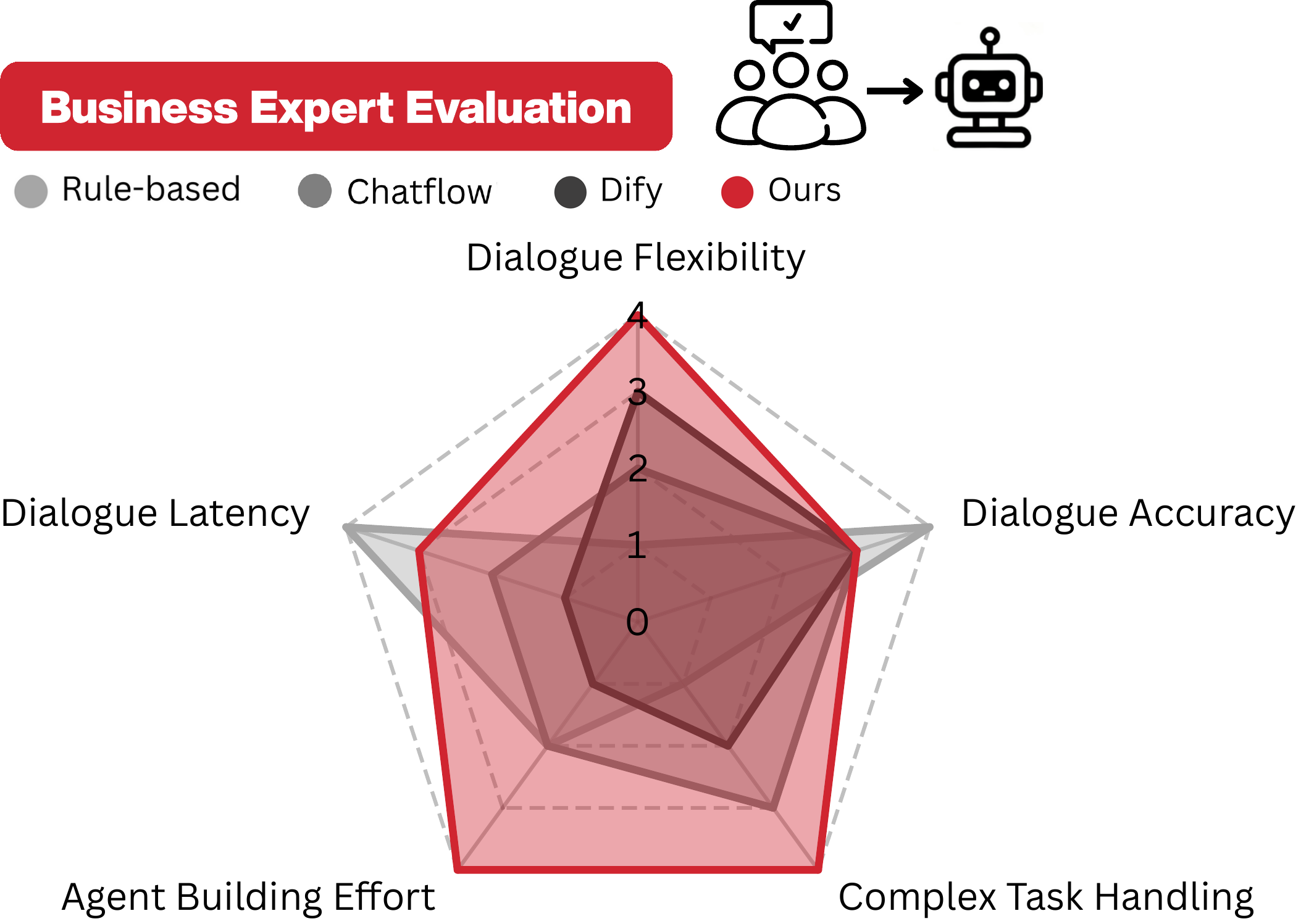}
    \vspace{-0.5em}
    \caption{Human evaluation of proactive outbound systems across five technical aspects (0 = worst, 4 = best).}
    \label{fig:human evaluation}
    \vspace{-1em}
\end{figure}

\subsection{Outbound Evaluation}

Existing evaluations offer limited coverage of outbound service scenarios, which feature more complex dialogue management and business requirements. To address this gap, we conducted a three-month field experiment deploying multiple service automation systems to perform real outbound interactions. The task focuses on collecting information on customers’ debt-repayment capacity, a realistic banking scenario that is traditionally labor-intensive and presents an urgent need for AI-driven automation.

\subsubsection{Implementations}

Building on the system architecture in Figure~\ref{fig:system overview}, we develop a proactive outbound service agent distilled from over 30,000 high-quality, complete human–human dialogues across four local banks in China. Both training and deployment are conducted in Mandarin, thereby complementing insights derived from English benchmarks. To protect data privacy, we adopt the proposed distributed distillation framework that generates service flowcharts on local servers before aggregating them for synthetic data generation and model fine‑tuning. This yields 12 unique TOFs and 1,800 synthetic dialogues generated by GPT-4, used to fine‑tune a smaller Qwen‑3‑8B model. We compare our system against three commercial baselines: a rule-based system driven by decision trees, a chatflow system, and an orchestrated agentic system built with Dify. Each system conducts 125 outbound communications for repayment prompting and multi-field data collection. Human experts assess the goal completion following \cite{thomson2024common}, with the \textit{Exemplary Rate} representing complete, high‑quality interactions and the \textit{Complete Rate} indicating minimal goal completion.

\subsubsection{Results and Discussions}
The results in Figure~\ref{fig:qualitative result} demonstrate that our proposed system outperforms existing commercial products, achieving the highest scores in Exemplary Rate and Completion Rate. We further evaluated their technical strengths and weaknesses through expert consensus. As shown in Figure~\ref{fig:human evaluation}, results confirm our system's top performance in dialogue flexibility (4), agent-building effort (4), and complex task handling (4), driven by its fully automated flowchart augmentation. In contrast, traditional systems driven by static workflow excel only in dialogue latency and accuracy, relying heavily on rule-based mechanisms with extensive human-defined knowledge. 

Dify's agent workflow provides moderate flexibility and accuracy but is limited by high latency, significant setup effort, and constrained complex task coverage. While such LLM operations (LLMOps) employ a drag-and-drop design to ease the development \cite{diaz2024large}, it becomes messy and hard to manage when business logic spans multiple scenarios, making it more suitable for early prototyping than robust, long-term solutions. It also requires manual configuration of node intents and transition strategies, increasing system setup and maintenance costs. In contrast, our flowchart approach delivers automated guidance that \textbf{prioritizes clear expression of business logic over micromanaging system execution}, demonstrating practical benefits in system development.

\subsection{Ablation Study}
We compare fine-tuning a LLaMA-3-8B-Instruct model on raw dialogue data versus a flowchart distillation approach to assess the memorization issue that leads to sensitive user preference leakage \cite{schwarzschild2024rethinking}. From MultiWOZ 2.0, we select 1,000 dialogues (200 per domain) and extract preference lists (e.g., price range, location, food preferences) from each dialogue as sensitive data. For our method, we generate 200 synthetic dialogues per domain following Section~\ref{sec:distillation}. Both models are trained under the same hyperparameters. Memorization of preference data is evaluated through extended slot prediction tasks \cite{ishihara-2023-training}, using slot masking for multi-class classification with provided candidate options. Results are reported across three metrics, compared to an untrained GPT-3.5 baseline: (i) Slot Prediction Accuracy; (ii) Top-3 Accuracy (exact match among top-3 predictions); and (iii) Multi-Slot Prediction Accuracy (exact sequence match for two masked slots).

\begin{figure}
    \centering
    \includegraphics[width=0.85\linewidth]{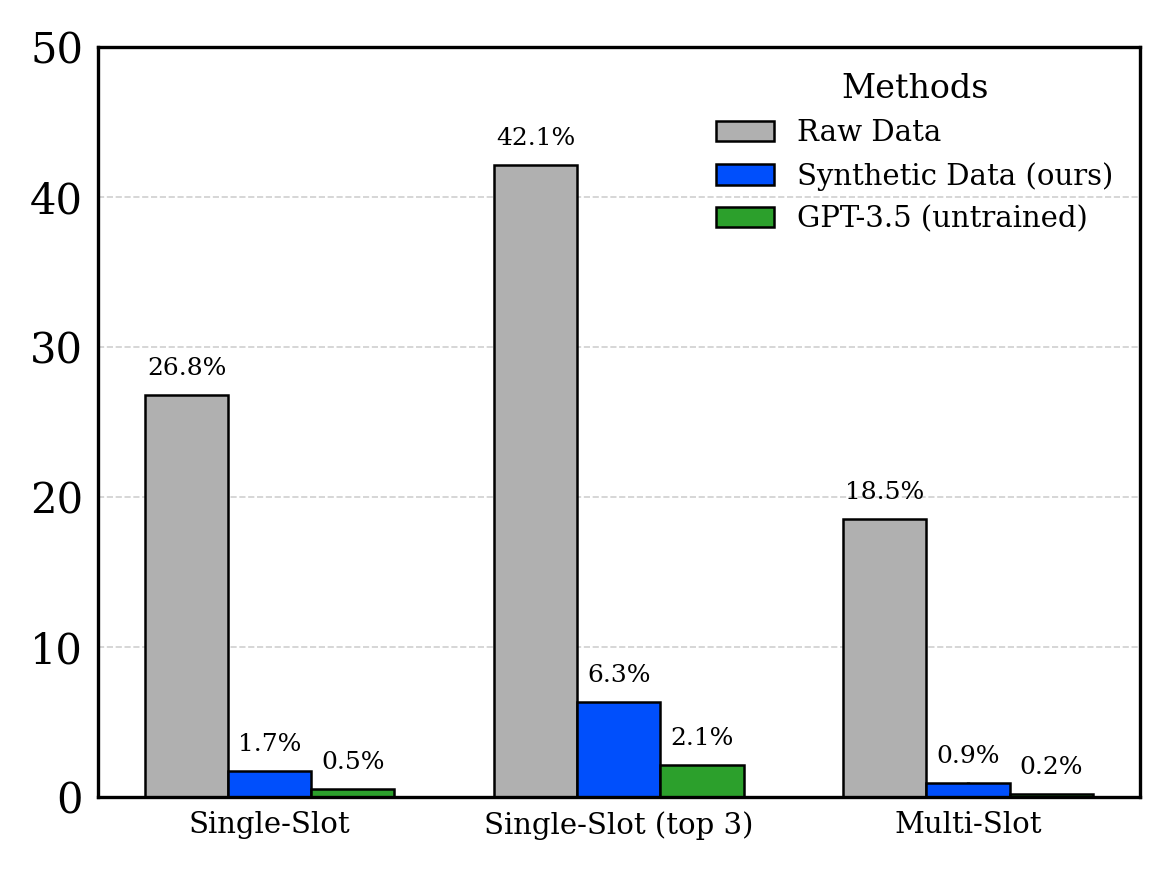}
    \vspace{-1.5em}
    \caption{Comparison of sensitive data memorization measured by user preference leakage (lower is better).}
    \label{fig:memorization}
    \vspace{-1.5em}
\end{figure}

Figure \ref{fig:memorization} reveals that fine-tuning on raw dialogue data results in significant memorization with frequent reproduction of sensitive preference information. In contrast, our approach substantially reduces these risks and closely aligns with the untrained baseline. This reduction arises because TOFs abstract only business logic, excluding fine-grained user details. By \textbf{prioritizing procedural structure over specific user data}, our method ensures robust privacy preservation without compromising task performance, offering a scalable solution for secure dialogue model training.

\begin{table}[!t]
\centering
\small
\caption{Case scenarios for system demonstration.}
\vspace{-0.5em}
\label{tab:cases}
\begin{tabular}{|c|p{5.3cm}|}
\hline
\textbf{Tasks} & \textbf{Description} \\
\hline
Verification & Validates correct, incorrect, or missing inputs with robust error handling. \\
\hline
Credit Limit & Guides users through credit limit increases via online banking or mobile apps. \\
\hline
Account Inquiry & Manages balance inquiries and unknown transaction investigations with accurate data retrieval and issue resolution. \\
\hline
Travel Notification & Facilitates the setup of overseas travel alerts with app guidance or direct assistance. \\
\hline
Open Questions & Addresses queries on wallet deposits, age limits, loan rates, and product recommendations.\\
\hline
\end{tabular}
\vspace{-0.5em}
\end{table}

\begin{figure}
    \centering
    \includegraphics[width=1\linewidth]{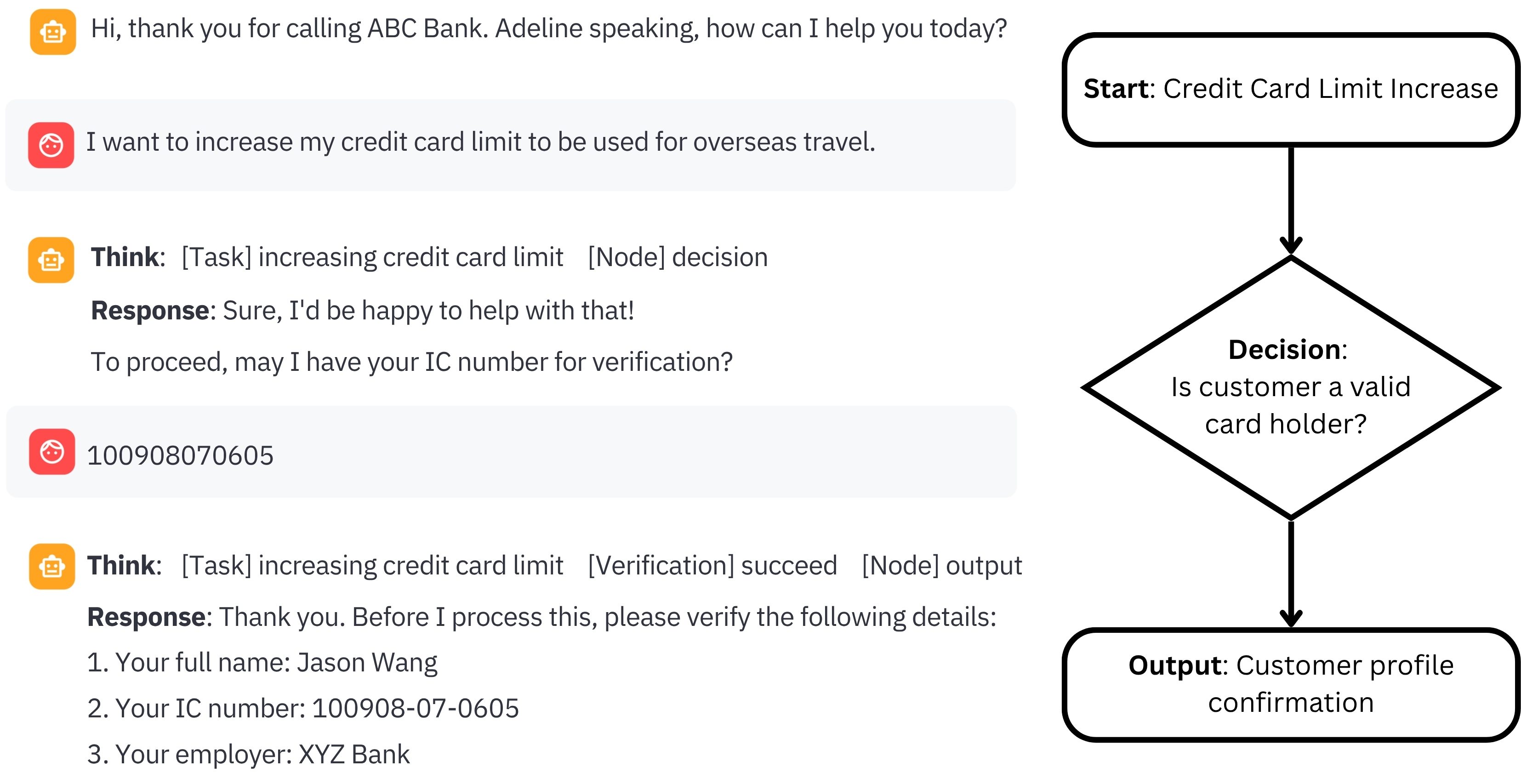}
    \vspace{-1.7em}
    \caption{System demonstration: flowchart-guided customer service automation for banking applications.}
    \vspace{-1.3em}
    \label{fig:demo}
\end{figure}

\subsection{Case Study and System Demonstration}

We validate the practical utility of the proposed approach through a web-based chatbot demonstration for online banking services. The system is driven entirely by service flowcharts for task coordination, eliminating the need for manual orchestration. Built on the GPT-3.5 model via the OpenAI API, it ensures broad accessibility with minimal computational requirements. The system supports response generation, operational management, and proactive product promotion before the dialogue ends. It features five distinct tasks, as detailed in Table~\ref{tab:cases}, each covering multiple scenarios to illustrate the system's adaptability in managing varied user interactions and information validity. Figure~\ref{fig:demo} presents an example service dialogue and its correspondence to the underlying flowchart.

\section{Conclusion}
This paper presents an innovative orchestration-free customer service automation framework utilizing Task-Oriented Flowcharts (TOFs) to deliver structured procedural guidance. By enhancing task coordination and enabling privacy-preserving model training, our approach outperforms baseline methods and commercial products in benchmark evaluations and demonstrates robust adaptability in real-world system deployments. Its scalable, interpretable design makes it ideal for integration into web platforms and enterprise systems, significantly improving operational efficiency while addressing the practical concerns of data privacy. Future research is encouraged to explore the potential of TOFs in multilingual and low-resource settings to further mitigate data scarcity and privacy challenges, alongside continued development of production systems based on the released demo project, advancing toward fully automated, streamlined web service automations.

\begin{acks}
The work was partially supported by the PolyU Start-up Fund (P0059983); the NSFC/RGC Joint Research Scheme (N\_PolyU5179/25); the Research Grants Council (HK) (PolyU25600624); and the Innovation Technology Fund (ITS/052/23MX and PRP/009/22FX).
\end{acks}

\bibliographystyle{ACM-Reference-Format}
\balance
\bibliography{reference}

\appendix
\newpage

\section{Example of Task-Oriented Flowchart in Mermaid Markdown and Chart}
\label{sec:flowchart demo}

We present a task-oriented flowchart for the common banking scenario of ``account inquiry,'' implemented in Mermaid Markdown and rendered in Mermaid Chart (see Figure \ref{fig:flowchart demo}). It systematically processes inquiries about account balances, transactions, and unrecognized charges. A Reflection node ensures user satisfaction by enabling loops to initial nodes for error handling and iterative refinement, enhancing the flowchart's robustness and user focus.

\vspace{0.5em}

\begin{lstlisting}
flowchart TD
    A["Start: Begin Customer Account Inquiry"] --> B{"Action/Decision: Determine Type of Enquiry"}
    B -- Account Balance/Credit Limit --> C{"Action/Decision: Is customer an active online banking/Connect App user?"}
    C -- No --> D["Output: Provide Assistance"]
    D --> E["Output: Display Balance/Limit Information"]
    C -- Yes --> F["Output: Would you like me to guide you through the app/online banking or do it for you?"]
    F -- Assistance --> D
    F -- Guidance --> G{"Action/Decision: App or Online Banking?"}
    G -- App --> H["Output: Provide App Guidance"]
    G -- Online Banking --> I["Output: Provide Online Banking Guidance"]
    B -- Transaction --> G
    B -- Unrecognized Charges --> N["Action/Decision: Perform Charge Check"]
    E --> R["Reflection: Confirm User Satisfaction"]
    H --> R
    I --> R
    N --> R
    R -- Satisfied --> J["End: Execute Closing Script"]
    R -- Not Satisfied --> B
\end{lstlisting}

\begin{figure}[!h]
    \centering
    \includegraphics[width=0.82\linewidth]{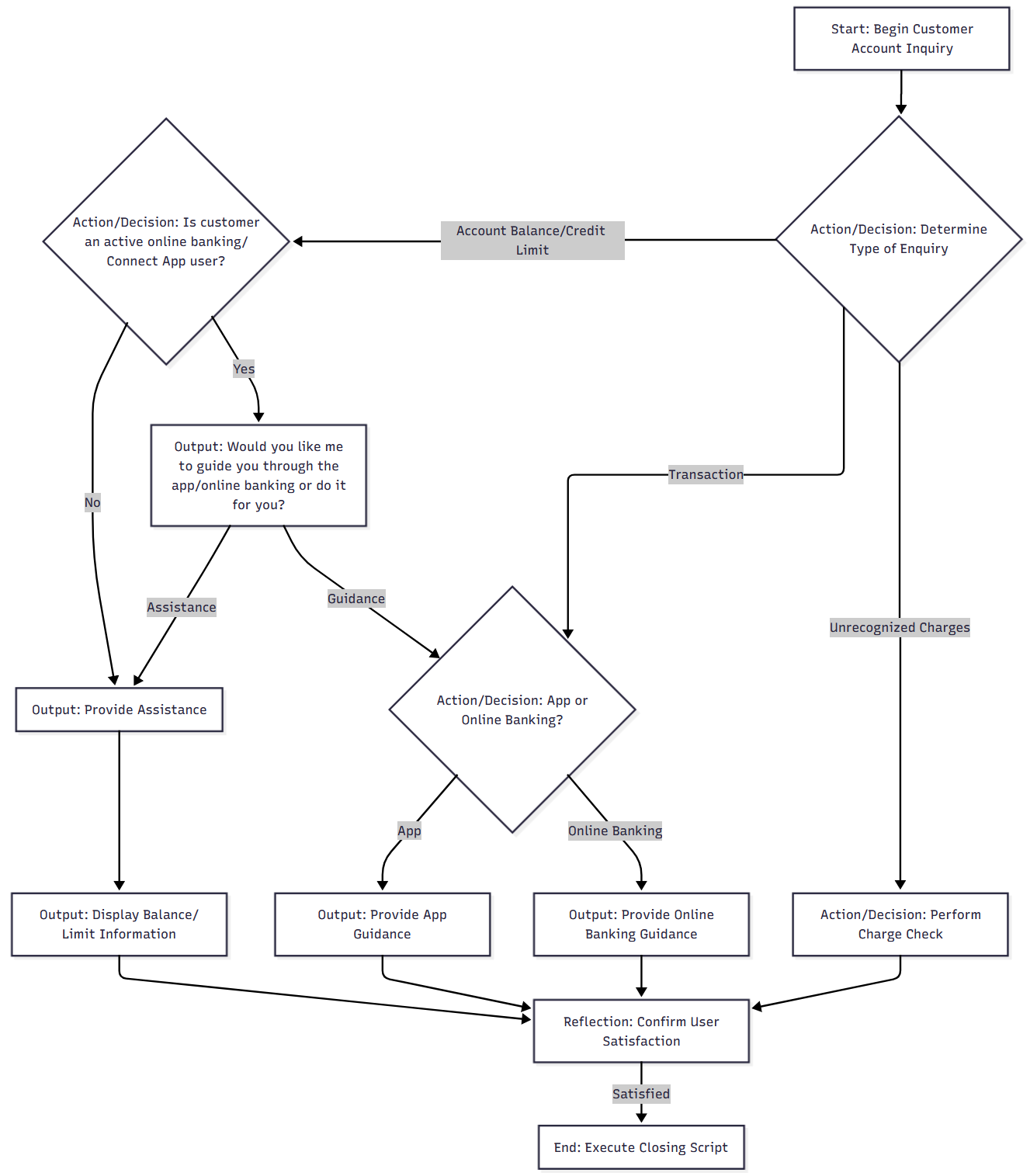}
        \vspace{-0.5em}
    \caption{Example of task-oriented flowchart.}
    \vspace{-1.5em}
    \label{fig:flowchart demo}
\end{figure}

\section{Proofs for WDIC Problem}
\label{sec:proofs}
\subsection{Submodularity of WDIC Objective Function}
\label{sec:submodularity}
\begin{theorem}
The WDIC coverage function
\[
f(S) = \left| \bigcup_{D_i \in S} S_i \right|
\]
is submodular and non-decreasing.
\end{theorem}
\begin{definition}[Submodularity]
A set function \( f: 2^N \to \mathbb{R} \) is \emph{submodular} if, for any \( A \subseteq B \subseteq N \) and any \( x \notin B \),
\[
f(A \cup \{x\}) - f(A) \ \ge \ f(B \cup \{x\}) - f(B).
\]
This property reflects \emph{diminishing returns}: the marginal benefit of adding an element decreases as the set grows.
\end{definition}
\begin{proof}
Let \( A \subseteq B \subseteq \mathcal{D} \) and \( D_j \notin B \).
The marginal gain of adding dialogue \( D_j \) to a set \( S \) is
\[
\Delta f(S, D_j) = \left| \bigcup_{D_i \in S \cup \{D_j\}} S_i \right| - \left| \bigcup_{D_i \in S} S_i \right|,
\]
i.e., the number of new intents covered by \( D_j \) that are not already in \( S \).
Since \( A \subseteq B \), the set of uncovered intents when adding \( D_j \) to \( B \) is a subset of those when adding it to \( A \). Hence,
\[
\Delta f(A, D_j) \ \ge \ \Delta f(B, D_j),
\]
which satisfies the submodularity condition.
Non-decreasingness follows directly: adding any dialogue can only increase, or leave unchanged, the total number of covered intents.
Therefore, \( f(S) \) is submodular and non-decreasing.
\end{proof}
\subsection{Approximation Guarantee for Greedy Algorithm}
\label{sec:greedy approximate bound}
Let \( f: 2^{\mathcal{D}} \to \mathbb{R}_{\ge0} \) be the monotone submodular coverage function, and let \( \mathbf{D'}^* \) denote the optimal set achieving full coverage under the given constraint.
For uniform costs (i.e., \( c_i = 1 \)), the greedy algorithm iteratively adds the element
\[
D_{t} \in \arg\max_{D \in \mathcal{D} \setminus \mathbf{D'}_{t-1}}
\frac{f(\mathbf{D'}_{t-1} \cup \{D\}) - f(\mathbf{D'}_{t-1})}{c(D)},
\]
selecting the dialogue that covers the most uncovered intents per unit cost (reducing to the most uncovered intents for uniform costs) until full coverage is achieved.
By submodularity, the greedy algorithm for the minimum set cover problem achieves an approximation ratio of \( H_n = \sum_{k=1}^n \frac{1}{k} \leq \ln n + 1 \), where \( n = |\mathcal{I}| \), as established in~\cite{wolsey1982analysis}.
For non-uniform costs, the same greedy strategy (maximizing marginal gain per unit cost) yields the same logarithmic approximation guarantee~\cite{chvatal1979greedy}.

Therefore,
\[
\sum_{D_i \in \mathbf{D'}_{\text{greedy}}} c_i \ \le\ (\ln |\mathcal{I}| + 1) \sum_{D_i \in \mathbf{D'}^*} c_i,
\]
establishing the \( (\ln |\mathcal{I}| + 1) \) bound, which is tight in the worst case under these conditions.

\section{Implementation Details}
\label{sec:implementation}

This section provides the implementation details of the experiments to ensure reproducibility. All components are developed in Python~3.10, with PuLP employed for solving the Integer Linear Program in dialogue selection and scikit-learn used for clustering in flowchart aggregation.

\subsection{Model Access and Deployment}
Proprietary LLMs, including GPT-3.5 and GPT-4, are accessed via the OpenAI API for synthetic data generation and benchmark evaluation. Locally deployed small language models, such as \href{https://huggingface.co/Qwen/Qwen3-8B-AWQ}{Qwen3-8B-AWQ} and
\href{https://huggingface.co/meta-llama/Meta-Llama-3-8B-Instruct}{Meta-Llama-3-8B-Instruct}, are hosted using the Hugging Face Transformers library on a single Nvidia V100 GPU with 32GB of memory. We have also evaluated on a single Nvidia RTX 4090 GPU using 8-bit quantization, which offers similar stability and improved memory efficiency.

\subsection{Inference Configuration}
All LLMs are evaluated with a temperature of 0 and a maximum generation length of 512 tokens, without the use of sampling mechanisms. For the proposed method, prompts consist of a base persona specification, the target task, and output format requirements, augmented with Task-Oriented Flowcharts represented in Mermaid Markdown. Baseline methods follow simple text instructions in accordance with the approach described by \cite{xu-etal-2024-rethinking}.

\subsection{Fine-tuning Settings}
Small language models are fine-tuned using supervised fine-tuning (SFT) with Low-Rank Adaptation (LoRA) under consistent hyperparameter settings. The learning rate is fixed at $1\times 10^{-5}$, with a batch size of~8 and an effective batch size of~32 achieved through four gradient accumulation steps. Training proceeds for three epochs, employing a cosine warmup schedule with a warmup ratio of~0.1. The LoRA configuration uses a rank of~8 and an alpha value of~32. Weight decay is set to~0.01, and optimization is performed using the AdamW algorithm with FP16 mixed-precision training.

\subsection{Human Annotation and Expert Evaluation Protocols}
To ensure the robustness and reliability of our experimental results, particularly in flowchart construction and system evaluation, we employed qualified human annotators and domain experts following established protocols for inter-annotator agreement \cite{thomson2024common}.

\subsubsection{Flowchart Construction Annotation}
For constructing human-annotated reference flowcharts (used as a baseline in Section 4.2), we recruited five professional annotators with expertise in NLP and business intelligence. Annotators were screened via a pilot task converting sample dialogues into TOFs, demonstrating high consistency in node-type assignments. Training comprised a 2-hour session on TOF components, Mermaid Markdown, and metrics (UMR and CPC, Section 3.2), focusing on intent-driven node assignment and domain adaptations. They independently annotated 200 dialogues (40 per MultiWOZ domain), resolving discrepancies via pairwise adjudication and majority vote. Inter-annotator agreement, measured by Cohen's kappa on node-type assignments, was 0.82, indicating substantial consistency and mitigating bias for a reliable human baseline.

\subsubsection{Deployed System Evaluation}
For the outbound evaluation (Section 4.3), we engaged ten banking domain experts to assess goal completion in 500 realistic outbound calls (125 per system). Collectively, they offer over 40 years of experience in customer service management, AI ethics, data privacy, and digital product development, with two holding certifications in CRM and financial auditing. A 1.5-hour training session covered system overviews, business requirements, and evaluation rubrics. Independent assessments were calibrated through discussion rounds, yielding Fleiss' kappa of 0.79 for Completion Rate. This diverse panel ensures comprehensive, real-world validation, underscoring the practical robustness of our flowchart-guided system against commercial benchmarks.

\subsection{Prompt Templates}
The prompts utilized for flowchart construction are provided for illustrative purposes. Each incorporates a few-shot example consisting of five input-output pairs to ensure consistent output formatting and to enhance comprehension of the tasks involved.

\begin{tcolorbox}[top=1pt, bottom=1pt, left=1pt, right=1pt]
\textbf{Domain Detection - }~\textit{Given customer utterance: ``\{customer\}'' and agent utterance: ``\{agent\}'', identify the primary domain (restaurant, hotel, attraction, train, taxi, or multi). Output only the domain name(s) comma-separated.}
\end{tcolorbox}

\begin{tcolorbox}[top=1pt, bottom=1pt, left=1pt, right=1pt]
\textbf{Intent Identification - }~\textit{Given customer utterance: ``\{customer\}'' and agent utterance: ``\{agent\}'' in {current domain} domain, identify the abstract intent (e.g., ``Inquire price requirement of restaurant''). Output only the intent description, concise and starting with a capitalized verb.}
\end{tcolorbox}

\begin{tcolorbox}[top=1pt, bottom=1pt, left=1pt, right=1pt]
\textbf{Node Type Selection - }~\textit{Given intent: ``{intent}'' in {current domain} domain, select a node type from the Task-Oriented Flowchart: start (initiates), prompt (requests input), decision (conditional, e.g., valid?), action (system query, e.g., retrieve), output (delivers info), reflection (evaluates state, e.g., goals met?), end (concludes). Output only the type (lowercase).}
\end{tcolorbox}

\noindent Additionally, the prompt associated with a LLM-based implementation of the utterance-to-node matching utility employed in flowchart evaluation is presented below:

\begin{tcolorbox}[top=1pt, bottom=1pt, left=1pt, right=1pt]
\textbf{Utterance-to-Node Matching - }~\textit{Given the following utterance and flowchart node descriptions, identify which node the utterance corresponds to based on the semantics. Output only the Node ID or `None'.}

\vspace{0.7em}
Utterance: {utterance}

\vspace{0.7em}
Nodes: \{node list\}
\end{tcolorbox}

\end{document}